\documentclass[12pt]{article}
\title{Faster online calibration without randomization:\\ interval forecasts and the power of two choices}
\usepackage[utf8]{inputenc}
\author{Chirag Gupta, Aaditya Ramdas
\\\\\texttt{chiragg@cmu.edu, aramdas@cmu.edu}
\\\\Carnegie Mellon University
}
\date{\today}
\usepackage[sort&compress]{natbib}
\PassOptionsToPackage{sort, compress}{natbib}
\setcitestyle{round}
\usepackage[margin=1in]{geometry}
\usepackage{amsthm}
\usepackage[ruled,vlined,linesnumbered]{algorithm2e}
\newtheorem{theorem}{Theorem}
\newtheorem{proposition}{Proposition}
\newtheorem{lemma}{Lemma}

\theoremstyle{definition}
\newtheorem{definition}{Definition}
\newtheorem{remark}{Remark}

\usepackage{amsmath} 
\usepackage{bbm}
\usepackage{bm}
\usepackage{xspace}
\usepackage{caption}
\usepackage{amssymb, amsmath}
\usepackage{mathtools}
\usepackage{enumitem}
\usepackage[utf8]{inputenc}
\usepackage{xcolor}
\usepackage{graphicx}
\usepackage{nicefrac}
\usepackage{centernot}
\usepackage{array}
\usepackage[toc,page]{appendix}
\usepackage{mathabx}
\usepackage{hyperref}
\usepackage{subcaption}
\usepackage{titlesec}
\usepackage{hhline}
\usepackage{float}
\usepackage{multirow}
\usepackage{multicol}

\newcommand{\Real}{\mathbb{R}}
\newcommand{\Xcal}{\mathcal{X}}
\newcommand{\Ycal}{\mathcal{Y}}
\newcommand{\Ccal}{\mathcal{C}}

\newcommand{\Exp}[2]{\mathbb{E}_{#1}\left\lbrack#2\right\rbrack}
\newcommand{\infnorm}[1]{\|#1\|_\infty}
\newcommand{\onenorm}[1]{\|#1\|_1}
\newcommand{\twonorm}[1]{\|#1\|_2}

\newcommand{\indicator}[1]{\mathbf{1}\{#1\}}
\newcommand{\abs}[1]{\left\lvert#1\right\rvert}
\newcommand{\roundbrack}[1]{\left( #1 \right)}
\newcommand{\ceil}[1]{\left \lceil{#1}\right \rceil}
\newcommand{\inner}[1]{\left\langle #1 \right\rangle}
\newcommand{\valp}{\text{Val$^\text{p}$}}
\newcommand{\vals}{\text{Val$^\star$}}
\newcommand{\val}{\text{Val}}

\begin{document}

\maketitle

\begin{abstract}%
  We study the problem of making calibrated probabilistic forecasts for a binary sequence generated by an adversarial nature. Following the seminal paper of Foster and Vohra (1998), nature is often modeled as an adaptive adversary  who sees all activity of the forecaster except the randomization that the forecaster may deploy. A number of papers have proposed randomized forecasting strategies that achieve an $\epsilon$-calibration error rate of $O(1/\sqrt{T})$, which we prove is tight in general. On the other hand, it is well known that it is not possible to be calibrated without randomization, or if nature also sees the forecaster's randomization; in both cases the calibration error could be $\Omega(1)$. Inspired by the equally seminal works on the power of two choices and imprecise probability theory, we study a small variant of the standard online calibration problem. The adversary gives the forecaster the option of making \textit{two nearby probabilistic forecasts}, or equivalently an interval forecast of small width, and the endpoint closest to the revealed outcome is used to judge calibration. This power of two choices, or imprecise forecast, accords the forecaster with significant power---we show that a faster $\epsilon$-calibration rate of $O(1/T)$ can be achieved even without deploying any randomization. 
\end{abstract}

\section{Introduction}
A number of machine learning and statistics applications rely on probabilistic predictions. In economics, the influential discrete choice framework uses probabilistic modeling at its core \citep{mcfadden1973conditional}. \citet{spiegelhalter1986probabilistic} argued that when predictive models are used in medicine for detecting disease, categorizing patient risk, and clinical trials, it is imperative that they provide accurate probabilities, in order to appropriately guide downstream decisions. Weather forecasters (and their audiences) would like to know the probability of precipitation on a given day \citep{brier1950verification}.

We study the problem of producing probabilistic forecasts for binary events, that are calibrated without any assumptions on the data-generating process. Informally, a forecaster is calibrated if, on all the days that the forecaster produces a forecast $p_t$ that is approximately equal to $p\in[0,1]$, the empirical average of the observations $y_t \in\{0,1\}$ is also approximately equal to $p$, and this is true for every $p \in [0,1]$ that is frequently close to a forecast \citep{dawid1982well}. We formalize this next.  %

\subsection{Calibration games and $\epsilon$-calibration}

\begin{minipage}[t]{0.45\linewidth}
\fbox{%
    \parbox{0.95\textwidth}{
        \begin{center} \textbf{Calibration-Game-I (classical)}\\(nature is an adaptive  adversary)
         \end{center} 
        At time $t = 1, 2, \ldots, $
        \begin{itemize}[itemsep=0.23cm]%
            \item Forecaster plays $u_t \in \Delta([0,1])$.
            \item Nature plays $v_t \in \Delta(\{0,1\})$.
            \item Forecaster predicts $p_t \sim u_t$.
            \item Nature reveals $y_t \sim v_t$.
        \end{itemize}
    }%
}
\end{minipage}
\begin{minipage}[t]{0.54\linewidth}
\fbox{%
    \parbox{0.95\textwidth}{%
        \begin{center} \textbf{Calibration-Game-II (POTC)}\\(nature is an adaptive  adversary, \\ forecaster has two nearby choices)
         \end{center} 
        Fix $ \epsilon > 0$. At time $t = 1, 2, \ldots, $
        \begin{itemize}%
            \item Forecaster plays  $p_{t0}, p_{t1} \in [0,1]$, such that $p_{t0} \leq p_{t1}$ and $\abs{p_{t1} - p_{t0}} \leq  2\epsilon$.
            \item Nature reveals $y_t \in \{0,1\}$.
            \item If $y_t = 1$, set $p_t = p_{t1}$;
            else set $p_t = p_{t0}$.%
        \end{itemize}
    }%
}
\end{minipage}
\vspace{0.1cm}

Calibration-Game-I models the problem as a game between a forecaster and nature. The forecaster produces a \emph{randomized} forecast $u_t \in \Delta([0,1])$, which is a distribution over the space of forecasts $[0,1]$. $\Delta(S)$ denotes the set of probability distributions over the set $S$ (in every case, $S$ is a standard set like $[0,1]$ with a canonical $\sigma$-algebra). Nature observes $u_t$ and responds with a Bernoulli distribution for the outcome $v_t \in \Delta(\{0, 1\}) = [0,1]$. We abuse notation slightly and use $v_t$ to denote both the Bernoulli distribution and its parameter in $[0,1]$. Then forecaster and nature draw their actual actions, the forecast $p_t\sim u_t$ and the outcome $y_t \sim v_t$, simultaneously. At time $T > 1$, the prior activities $(u_t, v_t, p_t, y_t)_{t = 1}^{T-1}$ are known to both players. The goal of the forecaster is to appear calibrated, defined shortly. Nature wishes to prevent the forecaster from appearing calibrated. Such a nature is typically referred to as an adaptive adversary. %

Even before defining calibration formally, we can see that randomization is essential for the forecaster to demonstrate any semblance of being calibrated. If the forecaster is forced to put all his mass on a single $p_t$ at each time (or equivalently if nature is an adaptive \textit{offline} adversary), nature can play $v_t = y_t = \indicator{p_t \leq 0.5}$ to render the forecaster highly miscalibrated \citep{oakes1985self, dawid1985comment}.

In anticipation of a forthcoming definition of $\epsilon$-calibration error (equation~\eqref{eq:actual-calibration-error}), we note that it will suffice for forecasters to only make discrete forecasts. Let $\epsilon > 0$ be a discretization or tolerance level, which is a small constant such as $0.1$ or $0.01$ depending on the application. For technical simplicity, we assume that $\epsilon = 1/2m$ for some integer $m \geq 2$. Consider the $\epsilon$-cover of $[0,1]$ given by the $m$ intervals $I_1 = [0, 1/m), I_2 = [1/m, 2/m), \ldots, I_m = [1-1/m,1]$. At time $t$, the forecaster makes a forecast on the `$2\epsilon$-grid' of the mid-points of these intervals: \[p_t \in \{M_1 := 1/2m = \epsilon, M_2 := 3/2m = 3\epsilon, \ldots, M_m := 1-1/2m = 1-\epsilon\}.\]
Denote the total number of times the forecast is $p_t = M_i$ until time $T \geq 1$ as  %
\[
N_i^T := \abs{\{t \leq T : p_t = M_i\}},
\]
and the observed average of the $y_t$'s when $p_t = M_i$ as
\[\ 
p^T_i := \begin{cases} \frac1{N^T_i}{\sum_{t \leq T: p_t = M_i} y_t}
~ ~ ~\hfill  \text{ if } N^T_i > 0, \\
 M_i \hfill \text{otherwise}.
\end{cases}
\]
Following \citet{foster1999proof}, the ($\ell_1$-)calibration error at time $T$, $\text{CE}_T$, is defined as the weighted sum of the prediction errors for each possible forecast:
\begin{equation*}
     \text{CE}_T := \sum_{i = 1}^m \frac{N_i^T}{T}\cdot \abs{M_i - p_i^T} \text{, or equivalently }\sum_{i = 1}^m \abs{\frac{1}{T}\sum_{t = 1}^T \indicator{p_t = M_i} (M_i - y_t)}. 
\end{equation*}
Finally, we define the $\epsilon$-calibration error $(\epsilon\text{-CE}_T)$ as the calibration error with a slack of $\epsilon$:
\begin{equation}
     \epsilon\text{-CE}_T := \max(\text{CE}_T - \epsilon, 0).  \label{eq:actual-calibration-error}
\end{equation}
In Calibration-Game-I, the forecaster and nature are allowed to randomize, thus $\epsilon\text{-CE}_T$ is a random quantity. A commonly studied object is its expected value. A forecaster is said to be $\epsilon$-calibrated if, for \emph{any} strategy of nature, the forecaster satisfies
\begin{equation}
    \lim_{T \to \infty} \Exp{}{\epsilon\text{-CE}_T} = 0, \text{ or equivalently } \Exp{}{\epsilon\text{-CE}_T} = \underbrace{o_T(1)}_{f(T)}. 
    \label{eq:eps-calibration-expectation}
\end{equation}
We are interested in the worst case value of $\Exp{}{\epsilon\text{-CE}_T}$ against an adversarial nature, denoted as $f(T)$, henceforth called the $\epsilon\text{-calibration rate}$ or simply calibration rate. %
We show results about the asymptotic dependence of $f(T)$ as $T \to \infty$, holding $\epsilon$ as a fixed problem parameter on which $f$ may depend arbitrarily. 

\subsection{Related work and our contributions}
\label{sec:main-lit-review}
A number of papers have proposed $\epsilon$-calibrated forecasting algorithms which guarantee $f(T) = O(1/\sqrt{T})$---the first was the seminal work of \citet{foster1998asymptotic}, followed by a number of alternative proofs and generalizations of their result (\citet{foster1999proof}, \citet{fudenberg1999easier}, \citet{vovk2005defensive}, \citet[Section 4.1]{mannor2010geometric}, \citet[Theorem 22]{abernethy2011blackwell}, \citet[Section 4.2]{perchet2015exponential}).

In Theorem~\ref{thm:slow-calibration}, we show that the $O(1/\sqrt{T})$ rate achieved by these algorithms is tight. There is a strategy for nature that ensures $f(T) = \Omega(1/\sqrt{T})$. Our proof uses a non-constructive lower bound for Blackwell approachability games \citep{mannor2013approachability}.  %

\citet{qiao2021stronger} recently showed that the worst-case calibration error without the $\epsilon$-slack, $\Exp{}{\text{CE}_T}$, is  $\Omega(T^{-0.472})$. %
In contrast, we treat $\epsilon$ as a small constant fixed ahead of time, and consider lower bounds on $\Exp{}{\epsilon\text{-CE}_T}$. Neither goal subsumes the other, so our lower bound complements theirs.  In particular, observe that $\Exp{}{\text{CE}_T} = \Omega(T^{-1/2})$ can be forced by nature by playing a non-adaptive Bernoulli strategy, drawing independently each $y_t \sim \text{Bernoulli}(p)$ for some $p$. This strategy seems insufficient for deriving a useful lower bound on $\Exp{}{\epsilon\text{-CE}_T}$. 
These comparisons are further discussed in Section~\ref{sec:lower-bound-previous}.

\citet{foster1999proof} showed that calibration is a Blackwell approachability instance (see Section~\ref{sec:calibration-approachability}), and while the rate $f(T) = \Omega(1/\sqrt{T})$ has not been formally established earlier (to the best of our knowledge), it is the rate one expects from a general Blackwell approachability instance \citep[Remark 7.7]{cesa2006prediction}. Instead, the community has looked to establish positive results for alternative notions: calibration with more stringent tests than $\epsilon$-calibration \citep{perchet2015exponential, rakhlin2011online}, calibration where the output space takes more than two values \citep{mannor2010geometric}, calibration with checking rules \citep{lehrer2001any, sandroni2003calibration, vovk2005defensive}, weak calibration \citep{kakade2004deterministic}, and smooth calibration \citep{foster2018smooth}. In particular, while no deterministic forecaster playing Calibration-Game-I can be $\epsilon$-calibrated, there exist deterministic forecasters who are weakly/smoothly calibrated \citep{foster2018smooth}. 

In our work, we take a slightly different approach from these papers. We retain the classical definition of $\epsilon$-calibration but change the calibration game. %
In Calibration-Game-II, also called the Power-Of-Two-Choices (POTC) game, the forecaster reveals two forecasts $p_{t0}, p_{t1} \in [0,1]$, such that $p_{t0} \leq p_{t1}$ and $\abs{p_{t1} - p_{t0}} \leq  2\epsilon$. Since the earlier binning scheme used a $2\epsilon$-grid, that is $1/m = 2\epsilon$, this effectively allows the forecaster to choose a full bin as their forecast (rather than its midpoint), or equivalently to choose two consecutive bin midpoints. Thus there is no randomization, and nature knows the two forecasts. If nature chooses to play $y_t = 0$, $p_t = p_{t0}$ is used to judge the calibration of the forecaster, and if nature chooses to play $y_t = 1$, $p_t = p_{t1}$ is used. (One could say that the forecast closer to reality is used for measuring calibration, or that the forecaster decides which one of the two forecasts to use; these are all equivalent.) Obviously, without the restriction of $p_{t0}, p_{t1}$ being $2\epsilon$-close, the problem is trivial: the forecaster would predict $p_{t0}=0$ and $p_{t1}=1$ in each round, and achieve zero error in every round. Requiring $\abs{p_{t1} - p_{t0}} \leq  2\epsilon$ makes the problem interesting. The POTC setup may appear surprising to some and we devote Section~\ref{sec:motivation} to motivating it. 

The summary of our main result (Theorem~\ref{thm:potc-calibration-fast}) is as follows. In the POTC game, the forecaster can ensure---deterministically---that
\begin{equation}
    \epsilon\text{-CE}_T = O(1/T). \label{eq:potc-fast-rate}
\end{equation}
Compared to~\eqref{eq:eps-calibration-expectation}, there is no expectation operator anymore since the forecaster is deterministic and nature being fully adaptively adversarial does not benefit from randomizing. 

Our forecaster is a variant of \citet{foster1999proof}. While Foster's forecaster randomizes over two nearby forecasts (making it \emph{almost deterministic} in the sense of \citet{foster2021forecast}), our forecaster predicts both these values and is judged with respect to the better one (and is actually, not almost, deterministic).

\begin{remark}[Generalization from binary to bounded outputs]
\label{rem:potc-game}
The POTC game can be modified for bounded, instead of binary, outputs. That is, nature can play $v_t \in [0,1]$ and calibration would be judged with respect to the average of the $v_t$'s on the instances when $p_t = M_i$. Note that this is not the same as nature playing $y_t \sim \text{Bernoulli}(v_t)$, since the calibration loss (left-hand-side of \eqref{eq:potc-fast-rate}) is not linear in $y_t$. With bounded outputs, the same $O(1/T)$ calibration rate can be achieved by a minor modification to our proposed forecasting strategy; see Appendix~\ref{subsec:potc-cal-bounded} for more details.  A similar remark holds for Calibration-Game-I and the corresponding lower bound of $\Omega(1/\sqrt{T})$. This latter fact is evident without further details since the lower bound can only increase if nature is given more flexibility.
\end{remark}

\textbf{Organization.} Section~\ref{sec:motivation} provides further context and motivation for the POTC game. Section~\ref{sec:calibration-potc} presents our algorithm for the POTC game and proves the fast calibration rate of $O(1/T)$ for it (Theorem~\ref{thm:potc-calibration-fast}). Section~\ref{sec:calibration-slow} reviews the well-known equivalence between calibration and Blackwell's approachability theorem \citep{blackwell1956analog}, using which we prove the slow calibration rate of $\Omega(1/\sqrt{T})$ for Calibration-Game-I (Theorem~\ref{thm:slow-calibration}). Most proofs are presented alongside the results. Section~\ref{sec:discussion} concludes with a discussion.%

\section{Motivation for the POTC calibration game}
\label{sec:motivation}
Calibration-Game-II or the POTC game is motivated by two rich fields of literature: imprecise probability and the power of two choices.

\subsection{A practical perspective via imprecise probability}
\label{sec:imprecise-probability}
The reader may wonder what the practical usefulness of the POTC game is. Why would we judge the forecaster in such a manner? The answer is that our earlier problem was phrased in a fashion that makes the connection to the power of two choices transparent. But one can also re-cast the problem in the language of \emph{imprecise probability}. In this area, one is typically not restricted to work with single, unique probability measures, but instead the axioms of probability are relaxed, and added flexibility is provided in order to work with \emph{upper} and \emph{lower} probability measures~\citep{walley1982towards}. 

In the context of our problem, instead of saying that the probability of rain is $0.3$, a forecaster is allowed to say $0.3 \pm \epsilon$. One may just say that the forecaster is slightly uncertain and does not wish to commit to a point forecast, and indeed we may not force a forecaster to announce a point forecast against their will. From a Bayesian or game-theoretic perspective, we may say that the forecaster allows bets against their forecast, represented as a contract which pays off $y_t$, but the forecaster's prices for buying and selling such a contract are slightly different. From a practical perspective,  this type of \emph{interval} forecast arguably has almost the same utility and interpretability to a layman as the corresponding point forecast. The use of upper and lower forecasts (translated to prices or betting odds) is standard in game-theoretic probability~\citep{shafer2019game}. Separately, the recent work of \citet{de2021randomness} establishes that \emph{randomness is inherently imprecise} in a formal sense, and provides a different justification for the use of interval forecasts for binary sequences. 

Remarkably, this small and seemingly insignificant change in reporting leads to a huge change in our ability to achieve calibration. (This gain can be rather puzzling: we were binning/gridding anyway, so why not report a full bin rather than its midpoint? How could that possibly improve our calibration error?!) Of course, we must figure out how to judge the quality of such an interval forecast: we must swap out $\indicator{p_t = M_i} (M_i - y_t)$ in $\text{CE}_T$ for a generalized notion of error that dictates how far $y_t$ was from the forecasted interval $A_t := [p_{t0},p_{t1}]$. To do this, we use the distance from a point to a convex set: we replace $M_i$ with the projection of $y_t$ onto $A_t$, denoted $\text{proj}(y_t,A_t)$. This is exactly what our POTC version does, just expressed differently. 

When we generalize the definition of calibration, it is notationally simpler to  restrict the forecasted interval endpoints to be the same $m$ gridpoints, meaning that $A_t = [M_i,M_{i+1}]$ for some $i \leq m-1$ (rather than $A_t = I_{i}$, the intervals whose midpoints are $M_i$). In this case, we call a method that produces interval forecasts $(I_t)_{t\geq 1}$ as being $\epsilon$-calibrated if:
\begin{equation}
    \max\roundbrack{\sum_{i = 1}^m \abs{\frac{1}{T}\sum_{t = 1}^T \indicator{\text{proj}(y_t,A_t) = M_i} \text{dist}(y_t,A_t)} - \epsilon, 0} = o(1), 
    \label{eq:interval-calibration}
\end{equation}
where $\text{dist}(y,A) := |y - \text{proj}(y,A)|$ is the distance of $y$ to interval $A$. (When the interval is a single point, we recover the original definition of calibration, but in this case we know that randomization is necessary for $\epsilon$-calibration.) 

Imprecise probability has also made an intriguing appearance in the simpler setting commonly considered in machine learning, of achieving calibration in offline binary classification in the presence of i.i.d. data. This is a problem where theoretical progress has been made on designing \emph{distribution-free} algorithms that have calibration guarantees by just assuming that the covariate-label pairs of data are i.i.d., while also performing well on real data~\citep{gupta2020distribution,gupta2021distribution,gupta2022top}. Venn predictors are a class of distribution-free algorithms that produce imprecise probability forecasts \citep{vovk2003self, vovk2014venn}. On observing the covariates of a new point, Venn predictors output a particular interval of probabilities $[p_0,p_1]$ for the unknown binary label. A strong, but slightly odd, calibration property holds: the authors prove that $p_Y$ (a random and unknown prediction, since $Y$ is unknown and random) achieves \emph{exact} calibration in finite samples. 
One can, in some sense, view our work as extending the use of such imprecise interval forecasts to the online calibration setting with adversarial data.

\subsection{The varied applications of the power of two choices}
\label{sec:POTC-review}
The power of two choices (POTC) refers to a remarkable result by \citet{azar1994balanced} for the problem of load balancing. Suppose $n$ balls are placed independently and uniformly at random into $n$ bins. It can be shown that with high probability, the maximum number of balls in a bin (the maximum \emph{load}) will be $\widetilde{\Theta}(\log n)$. %
Consider a different setup where the balls are placed sequentially, and for each ball, two bin indices are drawn uniformly at random and offered to a \emph{load-balancer} who gets to decide which of the two bins to place the ball in. The load-balancer attempts to reduce the maximum load by following a natural strategy: at each step, place the ball in the bin with lesser load. It turns out that with this strategy the maximum load drops exponentially to $\widetilde{\Theta}(\log \log n)$. 

The POTC result has led a number of applications. In a network where one of many servers can fulfil a request, it is exponentially better to choose two servers (instead of one) at random and allocate the server with fewer existing requests \citep{azar1994balanced}. Using two hashes instead of one significantly reduces the load of a single hash bucket \citep{broder2001using}. In circuit routing, selecting one of two possible circuits provably leads to decongestion \citep{cole1998randomized}. When allocating a task to one of many resources where an intensive query needs to be made about the resource capacity, querying two resources is often better than querying all resources, or querying a single resource \citep{azar1994balanced}.  Recently, \citet{dwivedi2019power} used the POTC to develop an online thinning algorithm that produces \emph{low-discrepancy sequences} on hypercubes, with applications to quasi Monte Carlo integration. For further applications and a survey of mathematical techniques, we refer the reader to the thesis of \citet{mitzenmacher1996power}, or the survey by \citet{mitzenmachar2001power}. 

In this paper, we find yet another intriguing phenomenon involving the POTC, this time in the context of calibration. We modify the classical setup of calibration (Calibration-Game-I) to the POTC setup (Calibration-Game-II), by offering the forecaster two nearby choices. We show that this change accords the forecaster with significant power, enabling faster calibration, even without randomization.

\section{Main results: algorithm and analysis}
\label{sec:calibration-potc}

Consider the POTC game (Calibration-Game-II). Recall that the forecaster's probabilities correspond to the mid-points of the intervals $I_1 = [0, 1/m), \ldots, I_m = [1-1/m,1]$, given by $M_1 = 1/2m, \ldots, M_m = 1-1/2m$. The forecaster can play either $(p_{t0}, p_{t1}) = (M_i, M_i)$ or $(p_{t0}, p_{t1}) = (M_i, M_{i+1})$ for some $i$. We can also say that the forecaster predicts one of the two intervals $\{M_i\}$ or $[M_i, M_{i+1}]$ respectively. %

We introduce some notation to describe the algorithm. For $i \in [m] := \{1, 2, \ldots, m\}$ and $t \geq 1$, define:
\begin{align*}
    \text{(left endpoint of interval $i$) }~& l_i = (i-1)/m, \\
    \text{(right endpoint of interval $i$) }~& r_i = i/m,\\
    \text{(frequency of interval $i$) }~& N^t_i = \abs{\{\indicator{p_s = M_i}: s \leq t\}},\\
    \text{(observed average when $M_i$ was forecasted) }~& p^t_i =\begin{cases} \sum_{s=1}^{t} y_s \indicator{p_s = M_i}/N_i^t ~~\text{ if } N^t_i > 0\\ \text{$M_i$~~~~~~~~~~~~~~~~~~~~~~~~~~~~~if $N^t_i = 0$},\end{cases}\\
    \text{(deficit) }~& d_i^t = l_i - p_i^t,\\
    \text{(excess) }~& e_i^t = p_i^t - r_i.
\end{align*}
The terminology `deficit' alludes to the fact that if $p_i^t$ is smaller than desired (to the left of $l_i$), then $d_i^t > 0$ ($p_i^t$ is `in deficit'). `Excess' has the opposite interpretation.

\noindent\fbox{%
    \parbox{\textwidth}{%
    \begin{center}
    \textbf{POTC-Cal: Algorithm for forecaster in Calibration-Game-II }\\(for notation, see Section~\ref{sec:calibration-potc} below)
    \end{center}
    \begin{itemize}
        \item 
    At time $t = 1$, play $(p_{10}, p_{11}) = (M_1, M_1)$. Thus $p_1 = M_1$.  
    \item At time $t+1$ ($t \geq 1)$, if
\begin{align*}
    \text{condition A: there exists an $i \in [m]$ such that $d_i^t \leq 0$ and $e_i^t \leq 0$,}
\end{align*} 
is satisfied, play $(M_i, M_i)$ for any $i$ that verifies condition A (that is, $p_{t+1}=M_i$). Otherwise, 
\begin{align*}
\text{condition B: there exists an $i \in [m-1]$ such that $e_{i}^t > 0$ and $d_{i+1}^t > 0$,}
\end{align*}
must be satisfied (see Lemma~\ref{lemma:cond-a-cond-b}). Play $(M_i, M_{i+1})$ for any index $i$ that verifies condition B (that is, $p_{t+1} = M_i$ if $y_{t+1} = 0$ and $p_{t+1} = M_{i+1}$ if $y_{t+1} = 1$).
\end{itemize}
    }%
}

\subsection{Forecasting algorithm}
The algorithm, presented on top of this page, is a variant of the one proposed by \cite{foster1999proof}. Foster's forecaster isolates two relevant $M_i$'s and randomizes over them; we use the same $M_i$'s to form the reported interval. %
At time $t+1$, if there is a forecast $M_i$ that is already `good' in the sense that $p_i^t \in [l_i, r_i]$, the forecaster predicts $M_i$. Otherwise, the forecaster finds two consecutive values $(M_i, M_{i+1})$ such that $p_i^t$ is in excess and $p_{i+1}^{t}$ is in deficit (such an $i$ exists by Lemma~\ref{lemma:cond-a-cond-b}, Appendix~\ref{appsec:supplementary}). The forecaster plays $(M_i, M_{i+1})$. If nature reveals $y_{t+1} = 0$, then $p_{t+1}=M_i$, and the excess of $p_i^t$ decreases. If nature reveals $y_{t+1} = 1$, then $p_{t+1}=M_{i+1}$, and the deficit of $p_{i+1}^t$ decreases.

\subsection{Analysis of POTC-Cal}

We now present our main result along with a short proof.
\begin{theorem}
\label{thm:potc-calibration-fast}POTC-Cal satisfies, at any time $T \geq 1$, for any strategy of nature,%
\begin{equation}
    \epsilon\text{-CE}_T \leq m/T. \label{eq:potc-fast-rate-again}
\end{equation}

\end{theorem}
\begin{proof}
Consider any $t \geq 1$. We write each of the $m$ terms in the calibration error at time $t$,  $\text{CE}_t$, as follows:
\begin{align*}
    \abs{\frac{1}{t}\sum_{s = 1}^t \indicator{p_s = M_i} (M_i - y_s)}  = \frac{N_i^t\abs{M_i - p_i^t}}{t} = \frac{ N_i^t (\epsilon + \max(d_i^t, e_i^t))}{t}.%
\end{align*}
Define $E_t^{(i)} :=  N_i^t \max(d_i^t, e_i^t)$, and observe that \[\epsilon\text{-CE}_T = \max\roundbrack{\sum_{i = 1}^m \frac{N_i^T\epsilon + E_T^{(i)}}{T} - \epsilon, 0} = \max\roundbrack{\sum_{i = 1}^m 
\frac{E_T^{(i)}}{T}, 0}.\] We will show that for every $i \in [m]$, $E_T^{(i)} \leq 1$, proving the theorem.

Consider some specific $i \in [m]$. If action $i$ is never played, then $N_i^T = 0$, and $E_T^{(i)} = 0$. Suppose an action $i$ has $N_i^T > 0$. For each $1 \leq t < T$, if $a_{t+1} \neq i$, then $E_{t+1}^{(i)} = E_{t}^{(i)}$. If $a_{t+1} = i$, then by Lemmas~\ref{lemma:proof-cond-a} and \ref{lemma:proof-cond-b} (stated and proved below), $E_{t+1}^{(i)} \leq %
\max(E_{t}^{(i)}, 1)$. In other words, at all $t$, the value of $E_{t+1}^{(i)}$ either stays bounded by $1$, or decreases compared to the previous value $E_{t}^{(i)}$. A trivial inductive argument thus implies $E_{T}^{(i)} \leq 1$. For completeness, we verify the base case: since $p_1 = M_1$, $E_1^{(i \neq 1)} = 0$ and $E_1^{(1)} \leq 1$ (as $d_1^1 \leq 0$ and $e_1^1 \leq 1$). %
\end{proof}

\begin{lemma}
Suppose condition A was satisfied at time $t+1$ and the forecast was $p_{t+1} = M_i$. Then, $N_i^{t+1}\max(d_i^{t+1}, e_i^{t+1}) \leq 1$. 
\label{lemma:proof-cond-a} 
\end{lemma}
\begin{proof}
Since $p_{t+1} = M_i$, $N_i^{t+1} = N_i^{t} + 1$, and $N_i^{t+1}p_i^{t+1} = N_i^t p_i^t + y_{t+1}$. Then, 
\begin{equation}
\begin{aligned}
    \abs{d_i^{t+1} - d_i^t} = 
    \abs{e_i^{t+1} - e_i^t} = 
    \abs{p_i^{t+1} - p_i^t} &= \abs{\frac{N_i^t p_i^t + y_{t+1}}{N_i^{t}+1} - \frac{N_i^t p_i^t + p_i^t}{N_i^{t}+1}}
    \\ &= \abs{\frac{y_{t+1} - p_i^t}{N_i^{t}+1}} \leq \frac{1}{N_i^{t}+1} = \frac{1}{N_i^{t+1}}.
\end{aligned}
\label{eq:small-change-in-errors}
\end{equation}
Since by condition A, $\max(d_i^{t}, e_i^{t}) \leq 0$, we obtain $\max(d_i^{t+1}, e_i^{t+1}) \leq 1/N_i^{t+1} $. 
\end{proof}

\begin{lemma}
Suppose condition A was not satisfied at time $t+1$ and the forecast was $p_{t+1} = M_i$, following condition B. Then  $N_i^{t+1}\max(d_i^{t+1}, e_i^{t+1}) \leq \max(N_i^t\max(d_i^{t}, e_i^{t}), 1)$. %
\label{lemma:proof-cond-b} 
\end{lemma}
\begin{proof}
Suppose $y_{t+1} = 0$. Since we are playing as per condition B, $e_i^t > 0$. Since $d_i^t + e_i^t = l_i - r_i = -1/m$, we have that $d_i^t < 0$. Thus,
\begin{equation*}
    y_{t+1} = 0 \implies \ e_i^t > 0 \text{ and } d_i^t < 0. 
\end{equation*}
Similarly, it can be verified that $y_{t+1} = 1 \implies e_i^t < 0\text{ and } d_i^t > 0$. Below we assume without loss of generality that $y_{t+1} = 0$. (A similar argument goes through for the case $y_{t+1} = 1$.)

We derive how $N_i^t \max(d_i^t, e_i^t)$ changes when going from $t$ to $t+1$. There are two cases:  $e_i^{t+1} \geq d_i^{t+1}$ or $e_i^{t+1} < d_i^{t+1}$. If $e_i^{t+1} \geq d_i^{t+1}$, then  %
\begin{align*}
    N_i^{t+1} \max(d_i^{t+1}, e_i^{t+1}) = N_i^{t+1} e_i^{t+1}  &= N_i^{t+1}p_i^{t+1} - N_i^{t+1}r_i 
    \\ &= N_i^t p_i^t - N_i^{t+1}r_i \quad \text{(since $y_{t+1} = 0$)}
    \\ &= N_i^te_i^t - r_i 
    \\ &= N_i^t\max(d_i^t, e_i^t) - r_i \leq N_i^t\max(d_i^t, e_i^t). 
\end{align*}
On the other hand if $e_i^{t+1} < d_i^{t+1}$, then, 
\begin{align*}
N_i^{t+1} \max(d_i^{t+1}, e_i^{t+1}) &= N_i^{t+1} d_i^{t+1} 
\\ &\leq N_i^{t+1} (d_i^t + \abs{d_i^{t+1} - d_i^t}) 
\\ &\overset{(*)}{<} N_i^{t+1} (0 + 1/N_i^{t+1}) = 1. 
\end{align*}
Inequality $(*)$ holds since $d_i^t < 0$ and $\abs{d_i^{t+1} - d_i^t} \leq 1/N_i^{t+1}$ (see set of equations \eqref{eq:small-change-in-errors}).%

\end{proof}

\section{$\Omega(1/\sqrt{T})$ lower bound for the classical calibration game}
\label{sec:calibration-slow}
Calibration-Game-I can be viewed as a repeated game with vector-valued payoffs/rewards. Such games were studied by \citet{blackwell1956analog}, and are now commonly referred to as Blackwell approachability games. We review the reduction from calibration to Blackwell approachability and use it to prove the lower bound. Throughout this section, we denote the action space of the forecaster as $\Xcal = \{M_1, M_2, \ldots, M_m\}$ and that of nature as $\Ycal = \{0, 1\}$. The random plays of the forecaster lie in $\Delta(\Xcal)$ which is a probability simplex in $m$ dimensions. We embed $\Delta(\Xcal)$ in $\Real^m$ to simplify discussion.

\subsection{Calibration as an instance of Blackwell approachability}
\label{sec:calibration-approachability}

The fact that calibration can be modelled as a Blackwell approachability instance is well-known (since \cite{foster1999proof, hart2000simple}). %
Suppose the actions of the forecaster and nature give a reward %
$r : \Xcal \times \Ycal \to \Real^m$ defined as follows: the $i$-th component of the reward vector $a  = r(p \in \Xcal, y\in \Ycal) \in \Real^m$ is given by \vspace{-0.1cm}
\begin{equation}
    a_i = \indicator{p = M_i} \cdot (M_i - y). \label{eq:reward-blackwell}
\end{equation} 
Let $\bar{a}^T := \sum_{i=1}^T r(p_t, y_t)/T$ be the average reward vector given component-wise by $\bar{a}^T_i = \sum_{t = 1}^T \indicator{p_t = M_i}(M_i - y_t)/T$. Let $B_\epsilon$ be the $\ell_1$-ball with radius $\epsilon$, and $\text{dist}$ the $\ell_1$-distance function. Note that
\[
\text{dist}(\bar{a}^T, B_\epsilon) = \max
\roundbrack{{\sum_{i = 1}^m \abs{\frac{1}{T}\sum_{t = 1}^T \indicator{p_t = M_i} (M_i - y_t)}} - \epsilon, 0} = \epsilon\text{-CE}_T. 
\]
Thus, the $\epsilon$-calibration condition \eqref{eq:eps-calibration-expectation} is equivalent to $\lim_{T \to \infty}\Exp{}{\text{dist}(\bar{a}^T, B_\epsilon)} = 0$. %
If this condition is satisfied, we say that $\bar{a}^T$ \emph{approaches} $B_\epsilon$ in the limit. %
\citet{blackwell1956analog} established necessary and sufficient conditions for approachability.

\begin{theorem}[Corollary to Theorem 3 by \citet{blackwell1956analog}]
\label{thm:response-approachability}
Assume the same setup as Calibration-Game-I, but the players receive a vector-valued reward $r$, as defined component-wise in \eqref{eq:reward-blackwell}. The forecaster can ensure that $\lim_{T \to \infty}\Exp{}{\text{dist}(\bar{a}^T, B_\epsilon)} = 0$ if and only if for every $v \in \Delta(\Ycal) = [0,1]$, there exists $u \in \Delta(\Xcal)$ such that the expected reward belongs to $B_\epsilon$: \vspace{-0.2cm}
\begin{equation}
    \forall v \in \Delta(\Ycal),\ \exists u \in \Delta(\Xcal) : \Exp{p \sim u, y \sim v}{ r(p, y)} = \sum_{i = 1}^m u_i ((1-v)\cdot r(M_i, 0) + v\cdot r(M_i, 1)) \in B_\epsilon. \label{eq:response-satisfiability}
\end{equation}
\end{theorem}
\noindent The hypothetical situation considered in the theorem is akin to a one-shot game but with the order of the players reversed: nature plays $v$ first and the forecaster responds with $u$. If the forecaster can \emph{respond} to every play by nature and ensure that the expected reward lies in $B_\epsilon$, then the forecaster can ensure that $\bar{a}^T$ approaches $B_\epsilon$ in the sequential game (where nature goes second each time).  \citet{abernethy2011blackwell} call this \emph{response-satisfiability}; in their words, Theorem~\ref{thm:response-approachability} is interepreted as response-satisfiability $\iff$ approachability. %

\begin{proposition}\label{prop:satisfiable-red-blackwell}
The forecaster can exhibit response satisfiability \eqref{eq:response-satisfiability}. Thus the forecaster playing Calibration-Game-I can ensure $\lim_{T \to \infty}\Exp{}{\text{dist}(\bar{a}^T, B_\epsilon)} = 0$ and be $\epsilon$-calibrated \eqref{eq:eps-calibration-expectation}. %
\end{proposition}

The proof of this well-known result is in Appendix~\ref{subsec:proof-prop:satisfiable-red-blackwell}. A second question is of the rate at which the expected reward vector approaches the desired set. As reviewed in the introduction, a number of papers have shown that the rate of approachability for the $\epsilon$-calibration game is $O(1/\sqrt{T})$. %
We show that this rate cannot be improved.%

\subsection{$\Omega(1/\sqrt{T})$ lower bound for the $\epsilon$-calibration error rate} 
\begin{theorem}
\label{thm:slow-calibration}
A forecaster playing Calibration-Game-I against an adversarial nature cannot achieve $\epsilon$-calibration at a rate faster than $O(1/\sqrt{T})$. That is, for every strategy of the forecaster, there is a strategy of nature that ensures  
\begin{equation}
    \Exp{}{\epsilon\text{-CE}_T} = \Omega(1/\sqrt{T}). \label{eq:eps-calibration-slow}
\end{equation}
\end{theorem}
\citet{mannor2013approachability} analyzed the convergence rate in approachability games and characterized conditions, which if satisfied by a target set $\Ccal$, entail that nature can ensure $\Exp{}{\text{dist}(\bar{a}^T, \Ccal)} = \Omega(1/\sqrt{T})$. In particular, if these conditions are satisfied $B_\epsilon$, Theorem~\ref{thm:slow-calibration} follows immediately since $\text{dist}(\bar{a}^T, \Ccal) = \epsilon\text{-CE}_T$. %

\begin{theorem}[Theorem 6.ii by \citet{mannor2013approachability}] 
\label{thm:slow-approachabilty}
Let $\Ccal$ be a closed convex set that is (i) minimal approachable, and (ii) mixed approachable. Then $\Ccal$ cannot be approached at a rate faster than $O(1/\sqrt{T})$, or in other words,  $\Exp{}{\text{dist}(\bar{a}^T, \Ccal)} = \Omega(1/\sqrt{T})$, where $\bar{a}^T$ is the average reward vector. 
\end{theorem}
\vspace{-3mm}

In what follows, we introduce the conditions (i) minimal approachability and (ii) mixed approachability in the context of our calibration game, and show that they are satisfied by $B_{\epsilon}$ (Lemma~\ref{prop:minimal-approachability} and Lemma~\ref{prop:mixed-approachability} respectively). %

\subsection{Minimal approachability}
For a point $u \in \Real^m$ and a convex set $K \subseteq \Real^m$, define the distance of $u$ from $K$ as $d_K(u) = \inf_{u' \in K}\twonorm{u - u'}$. For any $\lambda > 0$, a convex set $K' \subseteq K$ is said to be a $\lambda$-shrinkage of $K$ %
if $\{u : d_{K'}(u)\leq \lambda\} \subseteq K$. In the following definition of minimal approachability, we implicitly assume that the set of action sets of the players and the corresponding rewards have been fixed, and the goal is to characterize which convex sets are approachable and which are not. 
\begin{definition}
A set $K$ is \emph{minimal} approachable if $K$ is approachable, but no $\lambda$-shrinkage of $K$ is approachable.
\end{definition}

We now show the first condition required by Theorem~\ref{thm:slow-approachabilty}. 
\begin{lemma}
\label{prop:minimal-approachability}
The set $B_\epsilon$ is minimal approachable.
\end{lemma}
\begin{proof}
Proposition~\ref{prop:satisfiable-red-blackwell} shows that $B_\epsilon$ is approachable so it remains to show that the minimality condition holds. Let $K$ be a $\lambda$-shrinkage of $B_\epsilon$ for some $\lambda > 0$. We first argue that $K \subseteq B_{\epsilon - \lambda}$. Suppose this were not the case, that is, there exists $u \in K$ such that $u \notin B_{\epsilon - \lambda}$. %
By definition of the $\ell_1$-ball $B_{\epsilon - \lambda}$, this means that $\onenorm{u} > \epsilon - \lambda$. We will show that such a $u$ cannot belong to any $\lambda$-shrinkage of $B_\epsilon$, in particular it cannot belong to $K$, leading to a contradiction.

Consider the point $u' = u + (\lambda/\twonorm{u})u$. Note that $d_{K}(u') \leq \twonorm{u - u'} = \lambda$. Since $K$ is a $\lambda$-shrinkage of $B_\epsilon$, this implies that $u' \in B_{\epsilon}$, or $\onenorm{u'} \leq \epsilon$. On the other hand, we have,
\begin{align*}
    (\epsilon \geq)\ \onenorm{u'} &= \onenorm{u}(1 + \lambda/\twonorm{u})
                \\&\geq \onenorm{u}(1 + \lambda/\onenorm{u}) &\text{(for any vector $v \in \Real^m$, $\onenorm{v} \geq \twonorm{v}$)}
                \\&=\onenorm{u} + \lambda >(\epsilon - \lambda) + \lambda = \epsilon,
\end{align*}
which is a contradiction. Thus $K \subseteq B_{\epsilon - \lambda}$, as claimed. 

It follows that $K$ is approachable only if $B_{\epsilon - \lambda}$ is approachable. We now show that for every $\lambda > 0$, $B_{\epsilon - \lambda}$ is not approachable.
As in the proof of Proposition~\ref{prop:satisfiable-red-blackwell}, the $i$-th component of the reward vector is given by $u_i(M_i - v)$. Suppose $v = 1/m$. Then for every $M_i$, $\abs{M_i - v} \geq 1/2m = \epsilon$, by definition of $m$. Thus $\abs{u_i (M_i - v)} \geq u_i \epsilon$, and $\sum_{i=1}^m\abs{u_i (M_i - v)}  \geq \sum u_i\epsilon = \epsilon$. Equivalently, $\Exp{}{r(p, y)} \notin B_{\epsilon - \lambda}$. By Theorem~\ref{thm:response-approachability}, $B_{\epsilon - \lambda}$ is not approachable. 
\end{proof}
In order to describe the second condition required by Theorem~\ref{thm:slow-approachabilty} and show that it holds for $B_\epsilon$, we need additional technical setup. The following subsection serves this purpose. 

\subsection{Reducing approachability to scalar-valued games}
The vector-valued approachability game induces a number of scalar-valued min-max games, one for each direction in $\Real^m$. The value of these scalar games is closely connected to the question of approachability. 

Consider the approachability of $B_\epsilon$ with respect to individual directions, represented by arbitrary vectors $q \in \Real^m$ (for intuition, one may equivalently think of $q$ being direction vectors, those with $\ell_2$-norm equal to one, but this restriction is technically unnecessary; we stick to $q \in \Real^m$). Let $c \in B_\epsilon$ be such that $q$ belongs to the \emph{normal cone} of $B_\epsilon$ at $c$, that is, $\inner{c, q} = \sup_{c' \in B_\epsilon} \inner{c', q} = \epsilon \infnorm{q}$. We call such a pair $(c, q)$ as \emph{admissible}. 
Consider the following one-shot min-max game defined for every admissible $(c, q)$: %
\vspace{-0.2cm}
\begin{align*}
\val(c, q) %
&= \min_{u \in \Delta(\Xcal)} \max_{v \in [0,1]} \inner{\Exp{p \sim u, y \sim v}{r(p, y)} - c, q}
\\ &= \min_{u \in \Delta(\Xcal)}\max_{v \in [0,1]}\roundbrack{ \sum_{i = 1}^m u_i q_i((1-v)M_i + v(M_i - 1)) - \inner{c, q}}
\\ &= \min_{u \in \Delta(\Xcal)}\max_{v \in [0,1]}\roundbrack{ \sum_{i = 1}^m u_i q_i(M_i -v) - \epsilon\infnorm{q}}.
\end{align*}
To appreciate the relationship between the $\val(c, q)$ games and the $B_\epsilon$-approachability game, consider the following. Suppose the forecaster can guarantee \emph{one-shot approachability}, that is, there exists a fixed $u^\star \in \Delta(\Xcal)$ such that for every $v \in [0,1]$, $\Exp{p \sim u^{\star}, y\sim v}{r(p, y)} \in B_\epsilon$. By definition of the normal cone, for every admissible $(c, q)$, and any $c' \in B_\epsilon$, it holds that $\inner{c' - c, q} \leq 0$. In particular, $\Exp{p \sim u^\star, y \sim v}{r(p, y)} \in B_\epsilon$ is such a $c'$. It follows that for every admissible $(c, q)$, $\text{Val(c, q)} \leq 0$.

This observation does not hold in the reverse direction: even if $\val(c, q) \leq 0$ for every admissible $(c, q)$, one-shot approachability need not hold. The intuition is that the optimal $u$ for different $(c, q)$ can be different, and it is unclear how to merge them to achieve one-shot forecasting for the approachability game. In a remarkable result, \citet{blackwell1956analog} showed that the result does hold in the reverse direction for the \emph{repeated} approachability game (as opposed to the one-shot approachability game). 

\begin{theorem}[Theorem 1 by \citet{blackwell1956analog}]
A convex set $K$ is approachable if and only if for every admissible $(c, q)$, %
$\val(c, q) \leq 0$. 
\end{theorem}

This condition has also been termed as halfspace-satisfiability by \citet{abernethy2011blackwell}. The result was stated in the language of convex cones by \citet{mannor2013approachability}. Notice that for our problem, the min-max game does not depend on $c$, once we replace $\inner{c, q}$ with $\epsilon \infnorm{q}$. We thus simplify notation and index our games only by $q \in \Real^m$:
\begin{equation}
    \val(q) := \min_{u \in \Delta(\Xcal)}\max_{v \in [0,1]}\roundbrack{ \sum_{i = 1}^m u_i q_i(M_i -v) - \epsilon\infnorm{q}}. \label{eq:mixed-game}
\end{equation}
We know that $B_\epsilon$ is approachable (Proposition~\ref{prop:satisfiable-red-blackwell}) and hence by Blackwell's result, halfspace-satisfiability must hold. That is, for every $q \in \Real^m$, $\val(q) \leq 0$. For completeness, we verify this in Proposition~\ref{prop:halfspace-satisfiability} (Appendix~\ref{appsec:supplementary}).  

Having defined the $\val(q)$ games, we are now ready to define mixed approachability. 

\subsection{Pure$^\star$ game and mixed approachability}
In order to achieve a small value in the $\val(q)$ game, the forecaster may play a randomized strategy, that is, $u^\star \neq e_i$, where $e_i$ is one of the canonical basis vectors of $\Real^m$. On the other hand, since nature goes second, she has no incentive to randomize: there will exist an optimal strategy $v^\star \in \{0, 1\}$. In the following `pure' game, the forecaster is also not allowed to randomize over his actions. 
\begin{align}
\valp(q) := \min_{p \in \Xcal} \max_{y \in \{0, 1\}} \inner{r(p, y), q} - \epsilon \infnorm{q}. \label{eq:pure-game}
\end{align}
The superscript `p' in $\valp(\cdot)$ refers to the min-max game being over pure actions $p \in \Xcal$ and $y \in \{0,1\}$. Let us refer to this as the pure game, and the game \eqref{eq:mixed-game} as the mixed game. 

Suppose the approaching player (forecaster in our case) can ensure halfspace-satisfiability using only pure actions: $\forall q,\ \valp(q) \leq 0$. \citet{mannor2013approachability} showed that if this is true then then the approaching player can ensure approachability at a fast rate of $O(1/n)$. However, it is possible to achieve the fast rate even if the above condition is not true. %
Characterizing a situation where the fast rate is unachievable requires considering another game, whose value lies between the pure and mixed games. Define
\begin{align*}
\Xcal^\star &= \{p \in \Xcal : p \in \text{support}(u^\star) \text{, where $u^\star$ is some optimal mixed strategy for the forecaster}\};\\
\Ycal^\star &= \{y \in \{0, 1\} : y \in \text{support}(v^\star) \text{, where $v^\star$ is some optimal mixed strategy for nature}\}.
\end{align*}
Then define the pure$^\star$ game and its value as follows, 
\begin{equation}
    \vals(q) := \min_{p \in \Xcal^\star} \max_{y \in \Ycal^\star} \inner{r(p, y), q} - \epsilon \infnorm{q}. \label{eq:star-game}
\end{equation}
It can be shown that for any $q$, $\val(q) \leq \vals(q) \leq \valp(q)$ \citep{mannor2013approachability}. %
We now define mixed approachability. 
\begin{definition}
An approachable set is said to be mixed approachable if there exists a $q \in \Real^m$ such that while $\val(q) = 0$, $\vals(q) > 0$. 
\end{definition}
The following lemma witnesses a $q$ that shows that the mixed approachability condition required by Theorem~\ref{thm:slow-approachabilty} is satisfied. The witnessed $q$ in the proof is identified based on case (d) in the proof of Proposition~\ref{prop:halfspace-satisfiability}.
\begin{lemma}
\label{prop:mixed-approachability}
    There exists a $q \in \Real^m$ such that $\vals(q) > \val(q) = 0$. Thus $B_\epsilon$ is mixed approachable. 
\end{lemma}
\begin{proof}
Set $q_{1:m-1} = -1$ (i.e. $q_i = -1$ for all $i \in [m-1]$) and $q_m = 1$. Let us compute $\val(q)$. The game for nature reduces to maximizing $(\sum_{i=1}^{m-1}u_i - u_m)v$ which can be done by playing $v = \indicator{u_m \leq \sum_{i=1}^{m-1}u_i}$. We perform case work to identify the optimal play for the forecaster. 
\begin{itemize}[leftmargin=5mm]
    \item If $\sum_{i=1}^{m-1} u_i \leq u_m$, $v(u_m - \sum_{i=1}^{m-1}u_i) = 0$. The objective for the forecaster reduces to:
\[
\min_{u \in \Delta(\Xcal), \sum_{i=1}^{m-1} u_i \leq u_m} u_mM_m - \sum_{i=1}^{m-1}u_iM_i - \epsilon.
\]
Note that $0 < M_1 < M_2 < \ldots < M_m$. Thus the forecaster would set the minimum possible value of $u_m$, which under the constraints is equal to $0.5$. For the second term, the largest coefficient of a $u_i$ in the summation is $M_{m-1}$. Thus in order to minimize, the forecaster would set the maximum possible value of $u_{m-1}$, which under the constraints is $1 - u_m = 0.5$. We conclude that the minimum occurs at $u_m = u_{m-1} = 0.5$, $u_{i \notin\{m-1,m\}} = 0$. 
The objective value is equal to $0.5(M_m - M_{m-1}) - \epsilon = 0$.
\item If $\sum_{i=1}^{m-1} u_i \geq u_m$, $v=1$ is an optimal play for nature. The objective for the forecaster reduces to: 
\begin{align*}
\min_{u \in \Delta(\Xcal), \sum_{i=1}^{m-1} u_i \geq u_m} & u_m(M_m-1) - \sum_{i=1}^{m-1}u_i(M_i-1) - \epsilon .%
\end{align*}
As in the other case, this game is solved by observing that since $M_1 < M_2 < \ldots < M_m < 1$, the forecaster would want to put the maximum possible value on $u_m$ which is $0.5$ under the constraints. Among $u_{1:m-1}$, the multiplier of $(M_{m-1}-1)$ hurts the least. Thus the forecaster sets $u_{m-1} = 0.5$ and $u_{i \notin \{m-1, m\}}$. The objective value at the minimum is equal to $0$. 
\end{itemize}
In each case, the forecaster's optimal play is $u_m = u_{m-1} = 0.5$, $u_{i \notin\{m-1,m\}} = 0$. This essentially makes nature's action irrelevant; nature's optimal response is any $v \in [0,1]$. Thus we have shown that $\val(q) = 0$, $\Xcal^\star = \{M_{m-1}, M_m\}$, and $\Ycal^\star = \{0, 1\}$. Let us now compute 
\[
 \vals(q) = \min_{p \in \Xcal^\star} \max_{y \in \Ycal^\star}\roundbrack{ \indicator{p = M_m}(M_m-y) - \indicator{p = M_{m-1}}(M_{m-1}-y) - \epsilon}.
\] 
If the forecast is $p = M_m$, nature can respond $y = 0$ to achieve an overall objective of $1-2\epsilon > 0$. (We have assumed $m \geq 2$ so that $\epsilon < 0.5$.) If the forecast is $p = M_{m=1}$, nature can respond $y = 1$ to achieve an overall objective of $2\epsilon > 0$. Thus, $\vals(q) > 0$. 

\end{proof}

\subsection{Relationship to previous lower bounds for calibration}
\label{sec:lower-bound-previous}
\citet{qiao2021stronger} recently constructed a strategy of nature that ensures that the calibration error of any forecaster playing Calibration-Game-I satisfies 
    $\Exp{}{\text{CE}_T}
    = \Omega(T^{-0.472}).$ %
This bound is interesting on its own and neither weaker nor stronger than the bound we show in Theorem~\ref{thm:slow-calibration}. By studying $\Exp{}{\epsilon\text{-CE}_T}$, we allow the forecaster a slack of $\epsilon$ in his calibration error, which is standard in several earlier cited works (see Section~\ref{sec:main-lit-review}), and may be sufficient in practice given that the forecasts are themselves on a $2\epsilon$-grid. 

\citet{qiao2021stronger} also noted that $\Exp{}{\text{CE}_T} = \Omega(T^{-0.5})$ can be forced using a \emph{Bernoulli strategy}: at each time step, nature plays $y_t \sim \text{Bernoulli}(p)$ for some fixed $p \in [0, 1]$ unknown to the forecaster. However, in Appendix~\ref{appsec:bernoulli-strategy}, we provide initial (but not conclusive) evidence that the Bernoulli strategy seems insufficient to guarantee $\Exp{}{\epsilon\text{-CE}_T} = \Omega(T^{-0.5})$. We construct an $\epsilon$-calibrated forecaster that satisfies $\Exp{}{\text{CE}_T - \epsilon} \leq O(\text{poly}(\log T)/T)$ for the Bernoulli strategy ($\text{poly}(\log T)$  corresponds to polynomial terms in $\log(T)$). We conjecture that the stronger statement $\Exp{}{\epsilon\text{-CE}_T} = \Exp{}{\max(\text{CE}_T-\epsilon,0)} \leq O(\text{poly}(\log T)/T)$ also holds, which would mean that the Bernoulli strategy is insufficient to derive a $\Omega(1/\sqrt{T})$ bound on the $\epsilon$-calibration rate (as shown by Theorem~\ref{thm:slow-calibration}).

\vspace{-4mm}
\section{Summary}
\label{sec:discussion}
This paper connects three rich areas of the literature in a natural way: online calibration, the power of two choices, and imprecise probability. In summary, we show that by allowing the forecaster to output a deterministic short interval of probabilities (of length at most $2\epsilon$), we can achieve a faster rate of $O(1/T)$ for $\epsilon$-calibration against a fully adaptive adversarial reality who presents the binary outcome after observing the interval forecast. This should be compared to the $\Theta(1/\sqrt T)$ rate achievable with randomized point forecasts (the upper bound is a seminal result by~\cite{foster1998asymptotic}, the lower bound is ours), or the $\Theta(1)$ for deterministic point forecasts. 

Arguably, such narrow intervals are as practically interpretable as point forecasts, and since some sort of binning anyway underlies most calibration algorithms, it also feels natural to allow the forecaster to express their uncertainty in this fashion, especially since it avoids randomization and improves calibration. Thus, we view our work as a theoretical contribution with clear practical implications.

Several open questions remain, since we open a rather new line of investigation. We mention two: (a) lower bounds for our setting are unknown, and (b) we don't know if models providing more than two choices could possibly improve the rate further. We suspect that $1/T$ is the optimal rate since it corresponds to constant cumulative calibration error  (without normalization by $T$), which is incurred at $t = 1$ itself and seems unavoidable. %
Finally, it would be interesting to (c) figure out multidimensional analogs of our paper. We leave these for future work.%

We also note some POTC-style results in online learning. \citet{neu2020fast} show that for expert-based classification, providing the learner the choice to abstain from making a prediction improves the regret from $\Omega(1/\sqrt{T})$ to $O(1/T)$, similar to what we obtain in Theorem~\ref{thm:potc-calibration-fast}. In zero-order convex optimization or bandit convex optimization, allowing the learner two function evaluations enables the unknown gradient to be estimated using finite difference, leading to significantly improved rates (\citet{agarwal2010optimal} and follow-up work). For example, in the strongly convex case the rate improves from $\Omega(\sqrt{T})$ to $O(\log T)$. Such improvements also hold for non-smooth functions \citep{shamir2017optimal}. It would be interesting to consider POTC setups for multi-armed bandits (two arm-pulls instead of one) or expert-based online learning (choosing two experts instead of randomizing or choosing one expert).

\section*{Acknowledgements}
CG would like to thank the Bloomberg Data Science Ph.D. Fellowship for their generous support. We are grateful to Yusha Liu, Mingda Qiao, and the anonymous COLT reviewers, for feedback that helped improve a previous version of the paper. 

\bibliographystyle{plainnat}
\bibliography{references}

\appendix

\section{Supplementary results and deferred proofs}
\label{appsec:supplementary}

\subsection{Lemma~\ref{lemma:cond-a-cond-b} with proof}
\begin{lemma}
\label{lemma:cond-a-cond-b}
In POTC-Cal, for any $t \geq 1$, if condition A is not satisfied, then condition B must be satisfied. 
\end{lemma}
\begin{proof}
Note that for all $t$, $d_1^t \leq 0$ and $e_m^t \leq 0$, since $l_1 = 0$ and $r_m = 1$ (there cannot be a deficit for interval $1$ or an excess for interval $m$). If condition A does not hold for $i = m$, $d_m^t > 0$. Since $d_1^t \leq 0$ and $d_m^t > 0$, there exists an $i \in [m-1]$ such that, $d_i^t \leq 0$ and $d_{i+1}^t > 0$. If condition A does not hold, then $d_i^t \leq 0$ implies $e_i^t > 0$. Thus we have that $d^t_{i+1} > 0$ and $e^t_i > 0$; in other words, condition B holds at the exhibited $i$.
\end{proof}

\subsection{Proposition~\ref{prop:halfspace-satisfiability} with proof}
\begin{proposition}
\label{prop:halfspace-satisfiability}
    The forecaster in Calibration-Game-I can ensure halfspace-satisfiability. That is, for every $q \in \Real^m$, $\val(q) \leq 0$.
\end{proposition}
\begin{proof}
    Our construction is directly inspired by the proof of calibration by \citet{foster1999proof}. We perform a case analysis for different values of $q$:
    \begin{enumerate}[label=(\alph*)]
        \item 
    Suppose any $q_i = 0$. Then, playing $u_i = 1$ and $u_{j \neq i} = 0$ gives the objective value of $-\epsilon\infnorm{q} \leq 0$ irrespective of the value of $v$. %
    \item 
    Suppose $q_1 > 0$. Then, playing $u_1 = 1$ and $u_{j > 1} = 0$ gives the objective value as $q_1(M_1 - v) - \epsilon \infnorm{q} \leq q_1\epsilon -  \epsilon \infnorm{q} \leq 0$ (note that $M_1 = \epsilon$ and $v \geq 0$). 
    \item 
    Suppose $q_m < 0$. Then, playing $u_m = 1$ and $u_{j < m} = 0$ gives the objective value as $q_m(M_m - v) - \epsilon \infnorm{q} \leq \abs{q_m}\epsilon -  \epsilon \infnorm{q} \leq 0$ (note that $M_m = 1-\epsilon$ and $v \leq 1$). 
    \item 
    Suppose that neither of cases (a), (b), or (c) hold. Namely, $q_1 < 0$, $q_m > 0$ and $q_i \neq 0$ for any $i$. Let $j \in [m-1]$ be the smallest index such that $q_j < 0$ and $q_{j+1} > 0$. Then consider $u$ given by 
    \[
    u_j = \frac{\abs{q_{j+1}}}{\abs{q_{j}} + \abs{q_{j+1}}},\  u_{j+1} = \frac{\abs{q_j}}{\abs{q_{j}} + \abs{q_{j+1}}},\ u_{i \notin \{j, j+1\}} = 0.
    \]
    Observe two facts. First, $\sum_{i_1}^m u_i q_i v= (u_jq_j + u_{j+1} q_{j+1})v = 0$ since the value inside the brackets is itself equal to $0$ (any play $v$ of nature is thus rendered ineffective). Second, 
    \begin{align*}
        \sum_{i_1}^m u_i q_i M_i &= u_jq_jM_j + u_{j+1} q_{j+1}M_{j+1}
                                \\&= (u_jq_j + u_{j+1} q_{j+1})M_j +u_{j+1} q_{j+1}(M_{j+1} - M_j)
                                \\&= 0 + 2u_{j+1}q_{j+1}\epsilon
                                \\&= \frac{2\abs{q_j}\abs{q_{j+1}} \epsilon}{\abs{q_{j}} + \abs{q_{j+1}}}
                                \\&\leq \frac{\infnorm{q}\abs{q_{j}} \epsilon}{\abs{q_{j}} + \abs{q_{j+1}}} + \frac{\infnorm{q}\abs{q_{j+1}} \epsilon}{\abs{q_{j}} + \abs{q_{j+1}}}
                                \\&= \infnorm{q}\epsilon.
    \end{align*}
    Thus the overall objective value is at most $0$.
    \end{enumerate}
    The cases considered for $q$ are exhaustive, and in each case we verified that the forecaster can guarantee that the objective value is at most $0$. 
\end{proof}

\subsection{Proof of Proposition~\ref{prop:satisfiable-red-blackwell}}
\label{subsec:proof-prop:satisfiable-red-blackwell}

The $i$-th component of the expected reward vector \eqref{eq:response-satisfiability} is $u_i[(1-v)M_i + v(M_i - 1)] = u_i(M_i - v)$. Suppose $v \in I_j$. Consider playing $u \in \Delta(\Xcal)$ given by $u_j = 1$, $u_{i \neq j} = 0$. Then $\abs{\Exp{p \sim u, y \sim v}{ r(p, y)} } = \abs{\sum_{i=1}^m u_i(M_i - v)} = \abs{M_j - v}$. Since $v \in I_j$, $M_j$ is the mid-point of $I_j$, and the radius of each interval is $\epsilon$,  $\abs{M_j - v} \leq \epsilon$. Thus $\Exp{u, v}{r(p, y)} \in B_\epsilon$.

\section{Bernoulli strategy seems insufficient to prove an $\Omega(1/\sqrt{T})$ lower bound on the $\epsilon$-calibration rate}
\label{appsec:bernoulli-strategy}
This section is in the setup of Calibration-Game-I. Nature plays the `Bernoulli strategy': for a fixed $p\in[0,1]$ (unknown to the forecaster), nature plays $y_t \sim \text{Bernoulli}(p)$ each time. We describe a strategy for the forecaster that satisfies $\Exp{}{\epsilon\text{-CE}_T} = O(1/\sqrt{T})$ against a general strategy of nature; however, if nature is playing the Bernoulli strategy (instead of an arbitrary strategy), then our proposed strategy satisfies $\Exp{}{\text{CE}_T - \epsilon} \leq  O(\text{poly}(\log T)/T)$. Our result stops short of proving the following stronger statement against the Bernoulli strategy: $\Exp{}{\epsilon\text{-CE}_T} = \Exp{}{\max(\text{CE}_T - \epsilon, 0)} \leq O(\text{poly}(\log T)/T)$. We conjecture that this stronger statement holds as well, meaning that the Bernoulli strategy is insufficient to derive a $\Omega(1/\sqrt{T})$ bound on the $\epsilon$-calibration rate as shown in Theorem~\ref{thm:slow-calibration}.

\begin{figure}[t]
    \centering
\fbox{%
    \parbox{0.95\textwidth}{
        \begin{center} \underline{\texttt{PI-F99}: pre-initialized version of the $\epsilon$-calibrated strategy by \citet{foster1999proof}}
         \end{center}
         Fix $T_0 \in \mathbb{N}$. Set $T_k := 2^k T_0$, $K_k$ as in \eqref{eq:K-k}, $T^{(0)} := 0$, and $T^{(k)} := \sum_{j = 1}^{k-1} T_j = (2^k-1)T_0$.

         \vspace{0.5cm}
         For each time $t = 1, 2, \ldots $
         \begin{itemize}
            \item Identify smallest $k \in \mathbb{N}$ such that $t \leq T^{(k)}$. 
             \item Play \texttt{PI-F99$(T_{k-1})$} based on observations from $T^{(k-1)}+1$ until $t$: %
        \begin{itemize}[leftmargin=5mm]
                \item 
            if $t \leq T^{(k-1)} + mK_k$ (initialization phase): \\
            \hspace*{1cm}identify $j$ such that $t - T^{(k-1)} \in ( (j-1)K_k + jK_k]$ and predict $M_j$.
            \item if $t > T^{(k-1)} + mK_k$:\\\hspace*{1cm} follow Foster's strategy based on observations from $T^{(k-1)}+1$ until $t$.
        \end{itemize}
    \end{itemize}
    }
}%
\end{figure}

\subsection{The strategy}

The strategy we propose is a pre-initialized version of the $\epsilon$-calibration strategy of \citet{foster1999proof}---we call it \texttt{PI-F99} (for pre-initialized-Foster-99). %
\texttt{PI-F99} relies on a few constants: a large enough \emph{doubling} horizon $T_0 \in \mathbb{N}$, \emph{doubled} versions of $T_0$, namely $T_k =: 2^k T_0$ for $k \in \mathbb{N}$, and an initialization parameter 
\begin{equation}
    K_k := \ceil{(0.85\log T_k/\epsilon)^2 (\log \log (T_k/2) +  0.72 \log(5.2mT_k^2))} \label{eq:K-k}
\end{equation}
defined for each $k$. $K_k$ is the sufficient number of samples required to estimate the bias of $m$ Bernoulli random variables simultaneously and uniformly across time to a certain degree of reliability; further details become clear when analyzing. The constants in the definition of $K_k$ are not crucial (looser constants still lead to the same asymptotic dependence on $T$), but we identify them nevertheless in order to be precise.

\texttt{PI-F99} is a concatenation of certain sub-strategies \texttt{PI-F99}$(T_k)$ for $k \in \mathbb{N}$, each of which are strategies for Calibration-Game-I assuming the game only goes on until time $t = T_k$. The forecaster playing \texttt{PI-F99}$(T_k)$ first forecasts $p_t = M_i$, $K_k$ times each, for each $i \in [m]$ (thus until time $t \leq m K_k$). This is the \emph{initialization} phase. Then, for $t \in \{mK_k+1, \ldots, T_k\}$, the forecaster follows Foster's strategy initialized with \emph{current empirical frequencies} for each bin, based on what has been observed so far in the initialization phase. 

The actual strategy of the forecaster corresponds to a concatenation of \texttt{PI-F99}$(T_0)$, \texttt{PI-F99}$(T_1)$, \texttt{PI-F99}$(T_2)$, and so on. This is a version of the doubling trick \citep{cesa2006prediction}. The forecaster first plays \texttt{PI-F99}$(T_0)$ from $t = 1$ to $t = T_0$, then plays \texttt{PI-F99}$(T_1)$ from $t = T_0 + 1$ to $t = T_0 + T_1$, then \texttt{PI-F99}$(T_2)$ and so on. To be clear, when switching from \texttt{PI-F99}$(T_{k-1})$ to \texttt{PI-F99}$(T_k)$, the forecaster completely ignores the forecasts and observations so far, and restarts. The overall strategy is described in the box on top of the previous page. %

\subsection{Analysis}
Ignoring terms in $m$ and $\epsilon$ which are constants in $T_k$, $mK_k = O(\log^2 T_k) \ll \sqrt{T_k}$ (for sufficiently large $T_0$ and all $k \geq 0$). Assuming a worst case error of $1$ for each time until $T \leq mK_k$, we can show that the $\epsilon$-calibration error of \texttt{PI-F99}$(T_k)$ at any time $T \leq T_k$ satisfies the following:
\begin{align}
&\Exp{}{\max\roundbrack{\sum_{i = 1}^m \abs{\frac{1}{T}\sum_{t = 1}^T \indicator{p_t = M_i} (M_i - y_t)} - \epsilon, 0}} \nonumber \\ \leq~&\Exp{}{\max\roundbrack{\sum_{i = 1}^m \abs{\frac{1}{T}\sum_{t = 1}^{mK_k} \indicator{p_t = M_i} (M_i - y_t)}+ \sum_{i = 1}^m \abs{\frac{1}{T}\sum_{t = mK_k + 1}^T \indicator{p_t = M_i} (M_i - y_t)} - \epsilon, 0}} \nonumber \\
\leq~&\frac{mK_k}{T} + \Exp{}{\max\roundbrack{\sum_{i = 1}^m \abs{\frac{1}{T}\sum_{t = mK_k + 1}^T \indicator{p_t = M_i} (M_i - y_t)} - \epsilon, 0}} \nonumber \\
\leq~&\frac{O(\sqrt{T_k})}{T} + \Exp{}{\max\roundbrack{\sum_{i = 1}^m \abs{\frac{1}{T}\sum_{t = mK_k + 1}^T \indicator{p_t = M_i} (M_i - y_t)} - \epsilon, 0}}. \nonumber 
\end{align}
To bound the second term, note that Foster's strategy has a calibration rate of $O(1/\sqrt{T})$ starting with any initialization. Thus,
\begin{align}
\Exp{}{\max\roundbrack{\sum_{i = 1}^m \abs{\frac{1}{T}\sum_{t = mK_k + 1}^T \indicator{p_t = M_i} (M_i - y_t)} - \epsilon, 0}} &\leq%
\frac{%
O(\sqrt{T - mK_k})}{T} \leq \frac{C\sqrt{T_k}}{T}, \label{eq:pif99-k}
\end{align}
for some constant $C$ independent of $T_0$ and $k$. %
Thus for \texttt{PI-F99}$(T_k)$ at any time $T \leq T_k$, we have shown that $\Exp{}{\epsilon\text{-CE}_T} \leq O(1/\sqrt{T})$. From this, it follows that the overall strategy \texttt{PI-F99} satisfies $f(T) = O(1/\sqrt{T})$ asymptotically. This is shown in Proposition~\ref{ref:pi-f99-guarantee}, later in this subsection. Next, we perform the analysis for the Bernoulli strategy.

\begin{proposition}
\label{prop:bernoulli-rate}
\texttt{PI-F99} satisfies $\Exp{}{\text{CE}_T} \leq \epsilon + O(\text{poly}(\log T)/T)$ if nature follows the Bernoulli strategy.
\end{proposition}
\begin{proof}
Following the notation of Section~\ref{sec:calibration-potc}, let $N_i^T$ denote the number of times the mid-point $M_i$ is forecasted until time $T$. We have, %
\begin{align*}
    \Exp{}{\sum_{i = 1}^m \abs{\frac{1}{T}\sum_{t = 1}^T \indicator{p_t = M_i} (M_i - y_t)}} &=  \underbrace{\frac{\sum_{i \in [m] : \abs{p - M_i} > \epsilon}\Exp{}{ N_i^T\abs{M_i - \sum_{t=1}^T \indicator{p_t  = M_i}y_t/N_i^T}}}{T}}_{E_1} \\ &~~~ + \underbrace{\frac{\sum_{i \in [m] : \abs{p - M_i} \leq \epsilon}\Exp{}{ N_i^T\abs{M_i - \sum_{t=1}^T \indicator{p_t  = M_i}y_t/N_i^T}}}{T}}_{E_2}.
\end{align*}
We will show that  \[E_1 = O(\text{poly}(\log T)/T),\] 
and \[E_2 = \epsilon + O(1/T),\]
which will complete the argument. To this end, define $A^T$ as the number of times that the forecast is $\epsilon$-close to $p$, until time $T$:
\begin{equation}
    A^T := \sum_{i \in [m]: \abs{p - M_i} \leq \epsilon} N_i^T. \label{eq:def-At}
\end{equation} 
Lemma~\ref{lemma:A_t} shows that \texttt{PI-F99} satisfies $\Exp{}{A^T} = T - O(\text{poly}(\log T))$. This immediately leads to the bound for $E_1$; note that $\abs{M_i - \sum_{t=1}^T \indicator{p_t  = M_i}y_t/N_i^T} \leq 1$, and thus \[E_1 \leq \frac{\sum_{i \in [m] : \abs{p - M_i} > \epsilon}\Exp{}{ N_i^T}}{T}  = 1 - \frac{\Exp{}{A^T}}{T} = O(\text{poly}(\log T)/T).\]

\ifx false To bound $E_2$, we consider two cases separately: either (a) $p = r_j = l_{j+1}$ for some $j \in [m-1]$, that is, the bias of the Bernoulli is exactly at the common endpoint of two intervals; or (b) there is a unique index $j \in [m]$ such that $\abs{p - M_j} \leq \epsilon$. 

For case (b), without loss of generality, assume that $p > M_j$. Since $N_j^T \leq T$,  we get $E_2 \leq \mathbb{E}\abs{M_j - \sum_{t=1}^T \indicator{p_t  = M_j}y_t/N_j^T}$. \fi 

Bounding $E_2$ takes more work. The proof relies on the following `good' event occurring with high probability: 
\[G \equiv G_T: \text{for every $i \in [m]$, $\abs{p - \sum_{t=1}^T \indicator{p_t  = M_i}y_t/N_i^T} \leq \epsilon/2$.}\]
Due to the pre-initialization steps in \texttt{PI-F99}, it can be guaranteed that $\text{Pr}(G) = \text{Pr}(G_T) = 1 - O(1/T)$. For the details, we refer the reader to the proof of Lemma~\ref{lemma:A_t} (see case (a) in the proof), where we show a stronger version of this fact (namely, with $\epsilon/\log T_k$ instead of $\epsilon/2$), for \texttt{PI-F99}$(T_k)$ using a time-uniform concentration inequality (due to the time uniformity, the implication holds for \texttt{PI-F99} as well). \ifx false We break the expectation containing $E_2$ into two parts, using $\text{abs}(\cdot)$ instead of $\abs{\cdot}$ to avoid confusion with the conditioning operator: 
\begin{align*}
    T\cdot E_2 &\leq \sum_{i \in [m] : \abs{p - M_i} \leq \epsilon}\Exp{}{ N_i^T\text{abs}(M_i - \sum_{t=1}^T \indicator{p_t  = M_i}y_t/N_i^T) \mid G} + (1-\text{Pr}(G))\cdot T
    \\ &\leq \underbrace{\sum_{i \in [m] : \abs{p - M_i} \leq \epsilon}\Exp{}{ N_i^T\text{abs}(M_i - \sum_{t=1}^T \indicator{p_t  = M_i}y_t/N_i^T) \mid G}}_{E_3} + o(1). 
\end{align*}
Conditioned on the high probability event $G$, \fi 
We now do case work to bound $E_2$.
\begin{enumerate}[label=(\alph*),leftmargin=4pt]
    \item Suppose there exists an index $j \in [m]$ such that $\abs{p - M_j} \leq \epsilon/2$. This index must be unique since the $M_j$'s are $2\epsilon$ apart. Further, no $i \neq j$ can satisfy $\abs{p - M_i} \leq \epsilon$. We now obtain the following (below, we use $\text{abs}(\cdot)$ instead of $\abs{\cdot}$ to avoid confusion with the conditioning operator):
    \begin{align*}
         T\cdot E_2 &= \Exp{}{N_j^T\text{abs}(M_j - \sum_{t=1}^T \indicator{p_t  = M_j}y_t/N_j^T)}\\
         &\leq \Exp{}{  N_j^T\text{abs}(M_j - \sum_{t=1}^T \indicator{p_t  = M_j}y_t/N_j^T) \mid G} + (1-\text{Pr}(G))\cdot T
         \\ &= \Exp{}{ N_j^T\text{abs}(M_j - \sum_{t=1}^T \indicator{p_t  = M_j}y_t/N_j^T) \mid G} + O(1)
         \\ &\leq T\cdot \Exp{}{\text{abs}(M_j - \sum_{t=1}^T \indicator{p_t  = M_j}y_t/N_j^T) \mid G} + O(1)
        \\&\leq T\cdot\Exp{}{\abs{M_j - p} + \text{abs}(p - \sum_{t=1}^T \indicator{p_t  = M_i}y_t/N_j^T)\mid G} + O(1)
        \\&\leq T(\epsilon/2 + \epsilon/2) + O(1) = T\epsilon + O(1),
    \end{align*}
    where the inequality in the last line follows by the case assumption and the definition of $G$. Thus for this case, we have shown that $E_2 = \epsilon + O(1/T)$. 
\item Suppose $p$ is such that for every $i$, $\abs{p - M_i} > \epsilon/2$. To bound $E_2$, we are interested in the $M_i$'s for which $\abs{p - M_i} \leq \epsilon$. There can be at most two such $M_i$'s: an $M_j$ that satisfies $M_j \in (p+ \epsilon/2, p+\epsilon]$, and an $M_{l}$ that satisfies $M_j \in [p- \epsilon, p-\epsilon/2)$.%

Suppose there is an $M_j$ satisfying $M_j \in (p+ \epsilon/2, p+\epsilon]$. Set $R = \frac{1}{T}\sum_{t = 1}^T \indicator{p_t = M_j} (M_j - y_t)$ and note that $R \in [-1, 1]$. By Lemma~\ref{lemma:bdd-rv}, $\Exp{}{\abs{R}} \leq \Exp{}{R} + 2\cdot \text{Pr}(R < 0)$. Note that $\text{Pr}(R \geq 0) \geq \text{Pr}(G)$, since if $G$ holds, 
\begin{align*}
    R\cdot T &= \sum_{t=1}^T \indicator{p_t = M_j}(M_j - y_t) 
    \\&= \sum_{t=1}^T \indicator{p_t = M_j}((M_j - p) + (p - y_t))
    \\&\geq \sum_{t=1}^T \indicator{p_t = M_j}(\epsilon/2 + (p - y_t)) 
    \\&= N_i^T \epsilon/2 + \sum_{t=1}^T \indicator{p_t = M_j}(p - y_t)
    \\&\geq N_i^T \epsilon/2 - \abs{\sum_{t=1}^T \indicator{p_t = M_j}(p - y_t)}
    \\ &\geq N_i^T \epsilon/2 - N_i^T \epsilon/2 = 0,
\end{align*}
where the inequality in the last line is implied by $G$. Thus, $\text{Pr}(R < 0) \leq 1 - \text{Pr}(G) = O(1/T)$. Next, we bound $\Exp{}{R}$.
\begin{align*}
    \Exp{}{R} &= \frac{1}{T}\sum_{t = 1}^T \Exp{}{\indicator{p_t = M_j} (M_j - y_t)}
    \\ &= \frac{1}{T}\sum_{t = 1}^T \Exp{}{\indicator{p_t = M_j} (M_j - \Exp{}{y_t \mid (y_1, \ldots, y_{t-1}), (p_1, \ldots, p_{t})})}
    \\ &= \frac{1}{T}\sum_{t = 1}^T \Exp{}{\indicator{p_t = M_j} (M_j - p)}
    \\ &= \frac{\mathbb{E}N_j^T (M_j - p)}{T} \leq \frac{\mathbb{E}N_j^T}{T}\cdot \epsilon.
\end{align*}
Putting it together, we obtain $\Exp{}{\abs{R}} \leq \frac{\mathbb{E}N_j^T}{T}\cdot \epsilon  + O(1/T)$. 

Similarly, suppose there is an $M_{l}$ satisfying $M_l \in [p- \epsilon, p-\epsilon/2)$. For this $l$, define $S = \frac{1}{T}\sum_{t = 1}^T \indicator{p_t = M_l} (M_l - y_t)$. An identical argument as the one used for $R$ goes through; we use the inequality $\Exp{}{\abs{S}} \leq \Exp{}{-S} + 2\cdot \text{Pr}(S > 0)$ (from Lemma~\ref{lemma:bdd-rv}) and the relationship of $\text{Pr}(S > 0)$ to $G$ to obtain $\Exp{}{\abs{S}} \leq \frac{\mathbb{E}N_l^T}{T}\cdot \epsilon  + O(1/T)$. 

Finally, we conclude if both $M_k$ and $M_l$ with the given relationship to $p$ exist, then $E_2 \leq \Exp{}{\abs{R} + \abs{S}}$; if only $M_j$ exists, then $E_2 \leq \Exp{}{\abs{R}}$; if only $M_l$ exists, then $E_2 \leq \Exp{}{\abs{S}}$. In each case, \[E_2 \leq \frac{\Exp{}{A^T}\cdot \epsilon }{T} + O(1/T)\leq \epsilon + O(1/T). \]
\end{enumerate}
Since the two cases considered are exhaustive, this completes the proof. 

\ifx false 
On the other hand, let $M_j$ be such that that $p -\epsilon \leq M_j < p-\epsilon/2$  (there can be at most one such $M_j$). Set $S = \frac{1}{T}\sum_{t = 1}^T \indicator{p_t = M_j} (M_j - y_t)$ and note that $S \in [-1, 1]$. By Lemma~\ref{lemma:bdd-rv}, $\Exp{}{\abs{S}} \leq \Exp{}{-S} + 2\cdot \text{Pr}(S > 0)$. Now, 
\begin{align*}
    \Exp{}{-S} &= \frac{1}{T}\sum_{t = 1}^T \Exp{}{\indicator{p_t = M_j} (y_t - M_j)}
    \\ &= \frac{1}{T}\sum_{t = 1}^T \Exp{}{\indicator{p_t = M_j} (\Exp{}{y_t \mid (y_1, \ldots, y_{t-1}), (p_1, \ldots, p_{t})} - M_j)}
    \\ &= \frac{1}{T}\sum_{t = 1}^T \Exp{}{\indicator{p_t = M_j} (p - M_j)}
    \\ &= \frac{\mathbb{E}N_j^T (p - M_j)}{T}
    \\ &\leq \frac{\mathbb{E}N_j^T}{T}\cdot \epsilon.
\end{align*}

Now consider any $i$ such that $M_i < p-\epsilon/2$. Set $R_i = \frac{1}{T}\sum_{t = 1}^T \indicator{p_t = M_i} (M_i - y_t)$ and note that $R_i \in [-1, 1]$. By the above argument, $\Exp{}{\abs{R_i}} \leq \Exp{}{-R_i} + 2\cdot \text{Pr}(R_i > 0)$. 
The simpler 
 Then, the calibration error is upper bounded by  
\[
\frac{\Exp{}{A^T}\cdot \epsilon + \Exp{}{T - A^T}\cdot 1}{T} \leq \epsilon + \frac{T - \Exp{}{A^T}}{T}.
\]
\fi

\end{proof}

\begin{lemma}[for proving Proposition~\ref{prop:bernoulli-rate}]
\texttt{PI-F99} satisfies $\Exp{}{A^T} = T - O(\text{poly}(\log T))$, where $A^T$ is defined in the proof of Proposition~\ref{prop:bernoulli-rate}.
\label{lemma:A_t}
\end{lemma}
\begin{proof}
We first show \texttt{PI-F99}$(T_k)$ satisfies $\Exp{}{A^{T}} = T - O(\text{poly}(\log T))$ for $T \leq T_k$ once $k$ is large enough. We do so via two cases.

\begin{enumerate}[label=(\alph*)]
    \item For the first case, suppose $p = r_j = l_{j+1}$ for some $j \in [m-1]$, that is, the bias of the Bernoulli is exactly at the common endpoint of two intervals. In other words, $A^T = N_j^T + N_{j+1}^T$. We show that with high probability, the forecaster will \emph{learn} this index $j$ in the initialization phase of each \texttt{PI-F99}$(T_k)$, and continue playing either $M_j$ or $M_{j+1}$ until he switches to \texttt{PI-F99}$(T_{k+1})$.

Consider the strategy \texttt{PI-F99}$(T_k)$ for some $k \geq 0$. From time $t = mK_k$ onwards, each $M_i$ has been forecasted at least $K_k$ times, so that the value of $p_i^t$ is \emph{close} to $p$. To formalize close, we will use a time-uniform sub-Gaussian concentration inequality shown by \citet[equation~(3.4)]{howard2021time}. We use their inequality, replacing each instance of $t$ with $K_k/4$, since a Bernoulli is $(1/4)$-sub-Gaussian and each $M_i$ has been forecasted at least $K_k$ times. Additionally, we replace $\alpha$ with $1/mT_k^2$. It can be verified that the final deviation term inside the brackets is at most $\epsilon/\log T_k$; in other words, %
with probability at least $1 - 1/T_k^2$, the following `good' event occurs:
\[
\text{for all times $mK_k \leq t \leq T_k$, $\max_{i \in [m]} \abs{p_i^t - p} \leq \epsilon/\log T_k \leq \epsilon$}.
\]
The radius of each interval is $\epsilon$. Thus if the above event occurs, it follows that for intervals $i < j$, the right-endpoint $r_i < p - \epsilon \leq p_i^t$, so we have an excess ($e_i^T > 0$) until $T_k$; and for intervals $i > j+1$, the left-endpoint $l_i > p + \epsilon \geq p_i^t$, so we have a deficit ($d_i^T > 0$) until $T_k$. For interval $j$, either both $d_j^t, e_j^t \leq 0$ or $e_j^t > 0$; for interval $j+1$, either both $d_j^t, e_j^t \leq 0$ or $d_j^t > 0$. Overall,  with probability at least $1 - 1/T_k^2$,  for times $mK_k < t \leq T_k$, Foster's algorithm randomizes between $M_j$ and $M_{j+1}$ (possibly playing one of them deterministically).

\item The other case is when $p$ belongs to the interior of some interval $I_j$, $j \in [m]$, or $p \in \{0, 1\}$. Then, there exists some $\delta > 0$ such that $\abs{p - M_{i}} \geq \epsilon + \delta$ for all $i \neq j$. For a sufficiently large value of $\widetilde{k} \in \mathbb{N}$, $\delta > \epsilon/\log (T_{\widetilde{k}})$. Consider the strategy \texttt{PI-F99}$(T_k)$ for $k \geq \widetilde{k}$. As noted in the previous case, our choice of $K_{k}$ ensures that with probability at least $1 - 1/T_{k}^2$, for all times $mK_{k}  \leq t \leq T_{k}$, $\max_{i \in [m]} \abs{p_i^t - p} \leq \epsilon/\log T_{k} < \delta$. Using triangle inequality, we conclude that $\abs{p_j^t - M_{j}} \leq \epsilon$ and $\abs{p_i^t - M_{i \neq j}} > \epsilon$. It follows that for every $i \neq j$, there is either a deficit or an excess, and for $j$ there is neither. Thus with probability at least $1 - 1/T_k^2$, Foster's algorithm plays $M_j$ after time $mK_{k}$.
\end{enumerate}

Cases (a) and (b) lead to a lower bound on $\Exp{}{A^T}$ for \texttt{PI-F99}$(T_k)$, the expected number of times an $M_i$ is forecasted that is $\epsilon$-close to $p$. Namely, we obtain that for $k \geq \widetilde{k}$, for the strategy \texttt{PI-F99}$(T_k)$ that plays assuming a horizon of $T_k$ from $t = 1$ itself, we have for $T \leq T_k$: 
\begin{align}
    \Exp{}{A^T} &\geq (1 - 1/T_k^2)(T - mK_k) \nonumber
    \\ &\geq T(1 - 1/T_k^2) - O(\text{poly}(\log T_k)) \nonumber
    \\ &\geq T - %
    O(\text{poly}(\log T)). \label{eq:full-run-T_k}
\end{align}
The final inequality above holds since $T \leq T_k$.%

We derive the implication for the overall strategy that is actually played, \texttt{PI-F99}. Recall the notation $T^{(0)}= 0$ and $T^{(k)} = T_0 + T_1 + \ldots + T_{k-1}$. In \texttt{PI-F99}, the \texttt{PI-F99}$(T_k)$ strategy is played from time $T^{(k-1)} + 1$ to time $T^{(k-1)} + T_k = T^{(k)}$. Let $T$ be such that $T \in [T^{(k)} + 1, T^{(k+1)} ]$ for any $k \geq \widetilde{k}$. Then by \eqref{eq:full-run-T_k}, %
\begin{equation*}
    \Exp{}{A^T - A^{T^{(k)}}} \geq T - T^{(k)} %
    - O(\text{poly}(\log T)).
\end{equation*}
Again by \eqref{eq:full-run-T_k}, the above holds with $T \leftarrow T^{(k)}$, $T^{(k)} \leftarrow T^{(k-1)}$, if $k \geq \widetilde{k} + 1$ ($\leftarrow$ corresponds to replacing the term on the left with the term on the right):
\begin{align*}
    \Exp{}{A^{T^{(k)}} - A^{T^{(k-1)}}} %
     &\geq T^{(k)} - {T^{(k-1)}} %
    - O(\text{poly}(\log T)).
\end{align*}
Instantiating this recursively for all $k \geq \widetilde{k} + 1$, and adding the inequalities together gives us: 
\begin{equation*}
    \Exp{}{A^T}  \geq T - T^{(\widetilde{k})} - \log(T)\cdot O(\text{poly}(\log T)) = T - O(\text{poly}(\log T)),
\end{equation*}
since $\widetilde{k}$ is some fixed constant (given $p$). This completes the argument.

\end{proof}

\begin{lemma}
\label{lemma:bdd-rv}
For any bounded random variable $R \in [-a, a]$,
\begin{equation}
    \Exp{}{\abs{R}} \leq \min(\Exp{}{R} + 2a\cdot\text{Pr}(R < 0), \Exp{}{-R} + 2a \cdot \text{Pr}(R > 0)). 
\end{equation}
\end{lemma}
\begin{proof}
Note that,
\begin{align*}
    \Exp{}{\abs{R}} &= \Exp{}{R \cdot\indicator{R \geq 0} - R \cdot\indicator{R < 0}}
    \\ &= \Exp{}{R \cdot\indicator{R \geq 0} + R \cdot\indicator{R < 0} - 2R \cdot\indicator{R < 0}}
    \\ &= \Exp{}{R - 2R\cdot \indicator{R < 0}}
    \\ &\leq \Exp{}{R + 2a\cdot \indicator{R < 0}}
    \\ &= \Exp{}{R} + 2a\cdot\text{Pr}(R < 0).
\end{align*}
In the above proof, we can replace $R$ with $-R$ everywhere, since $-R \in [-a, a]$ as well. Thus we also obtain,
\[
\Exp{}{\abs{R}} = \Exp{}{\abs{-R}} \leq \Exp{}{-R} + 2a \cdot \text{Pr}(R > 0). 
\]
\end{proof}

\begin{proposition}
\label{ref:pi-f99-guarantee}
\texttt{PI-F99} achieves a calibration rate of $O(1/\sqrt{T})$ against any strategy of nature. 
\end{proposition}
\begin{proof}
Define $T^{(0)}= 0$ and $T^{(k)} = T_0 + T_1 + \ldots + T_{k-1} = (2^k - 1)T_0$, for $k \geq 1$. Further, define the cumulative (non-normalized) calibration error corresponding only to the times $t = t_1 + 1$ to $t_2$ as follows: 
\[
    \text{CE}(t_1, t_2) := \Exp{}{\max\roundbrack{\sum_{i = 1}^m \abs{\sum_{t = t_1 + 1}^{t_2} \indicator{p_t = M_i} (M_i - y_t)}, \epsilon(t_2 - t_1)}}. 
\]
From \eqref{eq:pif99-k}, \texttt{PI-F99} satisfies, for every $k\in \mathbb{N}$ and $T^{(k-1)} < t \leq T^{(k)}$, 
\begin{equation}
\text{CE}(T^{(k-1)}, t)  \leq \epsilon (t - T^{(k-1)}) + C\sqrt{T_{k-1}} \label{eq:fv-main-result}
\end{equation}
for some universal constant $C$ that does not depend on $k$. 

Now for a given $T > T^{(2)}$, let $k \geq 3$ be such that $T^{(k-1)} < T \leq T^{(k)}$. By triangle inequality and \eqref{eq:fv-main-result}, 
\begin{align*}
    \Exp{}{\max\roundbrack{\sum_{i = 1}^m \abs{\sum_{t = 1}^{T} \indicator{p_t = M_i} (M_i - y_t)}, \epsilon T}} &\leq \sum_{i=1}^{k-1} \text{CE}(T^{(i-1)}, T^{(i)}) + \text{CE}(T^{(k-1)}, T)
    \\ &\leq \sum_{i=1}^{k-1} (\epsilon(T^{(i)} - T^{(i-1)}) + C\sqrt{T_{i-1}}) \\ &\qquad\qquad\qquad+ \epsilon(T - T^{(k-1)}) + C\sqrt{T_{k-1}}
    \\ &= \epsilon T + \sum_{i=1}^{k} C\sqrt{T_{i-1}}
    \\ &= \epsilon T + C(\sqrt{T_0} + \sqrt{2T_0} + \sqrt{4T_0} + \ldots + \sqrt{2^{k-1}T_0})
    \\ &\leq \epsilon T + C\sqrt{T_0}\cdot \frac{\sqrt{2^{k}} - 1}{\sqrt{2} - 1}
    \\ &\leq \epsilon T + C\cdot \frac{\sqrt{2^kT_0}}{\sqrt{2} - 1}.
    \\ &\leq \epsilon T + C'\sqrt{T},%
\end{align*}
where $C' = C\cdot 2/(\sqrt{2}-1)$. The final inequality holds since for $k\geq 3$, \[\sqrt{2^k T_0} \leq \sqrt{4(2^{k-1}-1) T_0} =\sqrt{ 4 T^{(k-1)}} < 2\sqrt{T}.\] 
Dividing by $T$ and taking $\epsilon$ to the left-hand-size, we get that for all $T > T^{(2)}$, 
\[
\Exp{}{\max\roundbrack{\sum_{i = 1}^m \abs{\frac{1}{T}\sum_{t = 1}^{T} \indicator{p_t = M_i} (M_i - y_t)} - \epsilon, 0}} \leq C'\sqrt{T} = O(1/\sqrt{T}),
\]
as needed. 
\end{proof}

\section{Generalization of POTC-Cal to bounded outputs}
\label{subsec:potc-cal-bounded}
If the output is bounded instead of binary (see Remark~\ref{rem:potc-game}), then POTC-Cal can be modified as follows. The forecaster maintains $p_i^T$ as in the original algorithm, but these are now the mean of the $v_t$ values instead of $y_t$ values. The choice of the index $i$ and the final forecast $(p_{t0}, p_{t1})$ is made identically to the original POTC-Cal. Finally, the forecaster plays \begin{equation}
    \text{$p_t = p_{t0}$ if $v_t \leq r_i$, and $p_t = p_{t1}$ if $v_t > r_i$.} \label{eq:bounded-game-strategy} 
\end{equation}
Note that $r_i= l_{i+1}$ is the right (left) endpoint of interval $i$ (interval $i+1$), and thus a natural threshold for deciding which of the two intervals to play.

Lemmas~\ref{lemma:proof-cond-a} and \ref{lemma:proof-cond-b} hold for this modified setup and algorithm, and thus the $O(1/T)$ rate showed by Theorem~\ref{thm:potc-calibration-fast} also holds. Lemma~\ref{lemma:proof-cond-a} goes through since the set of equations~\eqref{eq:small-change-in-errors} can be modified as follows:
\begin{equation*}
    \abs{d_i^{t+1} - d_i^t} = 
    \abs{e_i^{t+1} - e_i^t} = \abs{\frac{v_{t+1} - p_i^t}{N_i^{t}+1}} \leq  \frac{1}{N_i^{t+1}}.
\end{equation*}

In the proof of Lemma~\ref{lemma:proof-cond-b}, we assumed without loss of generality that $y_{t+1} = 0$. This assumption can be modified to $v_{t+1} \leq r_i$ in keeping with the forecaster's updated strategy \eqref{eq:bounded-game-strategy}. The case $e_i^{t+1} < d_i^{t+1}$ goes through since it is a consequence of the set of equations \eqref{eq:small-change-in-errors}. For the case $e_i^{t+1} \geq d_i^{t+1}$, we have 
\begin{align*}
    N_i^{t+1} \max(d_i^{t+1}, e_i^{t+1}) &= N_i^t p_i^t + v_{t+1} - N_i^{t+1}r_i 
    \\ &= N_i^te_i^t + (v_{t+1}- r_i) \leq N_i^t\max(d_i^t, e_i^t),
\end{align*}
where the last inequality follows by the case assumption $v_{t+1} \leq r_i$. 

\end{document}